\documentclass{article}

\usepackage{ijcai23}

\usepackage{times}
\usepackage{soul}
\usepackage{url}
\usepackage{xr-hyper}
\usepackage[hidelinks]{hyperref}
\usepackage[utf8]{inputenc}
\usepackage[small]{caption}
\usepackage{graphicx}
\usepackage{amsmath}
\usepackage{amsthm}
\usepackage{booktabs}
\usepackage{algorithm}
\usepackage{algorithmic}
\usepackage[switch]{lineno}
\usepackage{bbm}
\usepackage{bm}
\usepackage{MnSymbol}
\usepackage{subfigure}
\usepackage{xcolor}

\newtheorem{example}{Example}
\newtheorem{theorem}{Theorem}

\theoremstyle{definition}
\newtheorem{definition}{Definition}

\newcommand{\etal}{\textit{et al. }}
\renewcommand{\Pr}{\mathrm{Pr}}

\makeatletter\newcommand*{\addFileDependency}[1]{
    \typeout{(#1)}
    \@addtofilelist{#1}
    \IfFileExists{#1}{}{\typeout{No file #1.}}
}\makeatother

\title{Explainable Reinforcement Learning via a Causal World Model}

\author{
    Zhongwei Yu$^1$
    \and
    Jingqing Ruan$^1$\And
    Dengpeng Xing$^1$\thanks{Corresponding author.}
    \affiliations
    $^1$Institute of Automation, Chinese Academy of Sciences\\
    \emails
    \{yuzhongwei2021, ruanjingqing2019, dengpeng.xing\}@ia.ac.cn
}

\begin{document}

\maketitle

\begin{abstract}
Generating explanations for reinforcement learning (RL) is challenging as actions may produce long-term effects on the future. In this paper, we develop a novel framework for explainable RL by learning a causal world model without prior knowledge of the causal structure of the environment. The model captures the influence of actions, allowing us to interpret the long-term effects of actions through causal chains, which present how actions influence environmental variables and finally lead to rewards. Different from most explanatory models which suffer from low accuracy, our model remains accurate while improving explainability, making it applicable in model-based learning. As a result, we demonstrate that our causal model can serve as the bridge between explainability and learning.
\end{abstract}

\section{Introduction}
Many real-world applications like finance and healthcare require AI systems to be well understood by users due to the demand for safety, security, and legality \cite{gunning_xaiexplainable_2019}. Aiming to help people better understand and work with AI systems, the field of Explainable AI (XAI) has recently attracted increasing interest from researchers. For example, a number of explanatory tools have been developed to pry into the black box of deep neural networks \cite{bach_pixel-wise_2015,selvaraju_grad-cam_2020,wang_shapley_2021}.

However, the domain of explainable reinforcement learning (XRL) has been neglected for a long time. Many XRL studies adopt classic tools of XAI such as saliency maps \cite{nikulin_free-lunch_2019,joo_visualization_2019,shi_self-supervised_2021}. These tools are not designed for sequential decision-making and are weak in interpreting the temporal dependencies of RL environments. Therefore, some studies investigate explaining specific components of the decision process, e.g., observations \cite{koul_learning_2018,raffin_decoupling_2019}, actions \cite{fukuchi_autonomous_2017,yau_what_2020}, policies \cite{amir_highlights_2018,coppens_distilling_2019}, and rewards \cite{juozapaitis_explainable_2019}. However, these studies rarely combine explanations with the dynamics of environments, which is important for understanding the long-term effects produced by agents' actions. In addition, real-world environments usually contain dynamics unknown to users, making it crucial to interpret these dynamics using explanatory models. Model-based RL (MBRL) uses predictive world models \cite{nagabandi_neural_2018,kaiser_model-based_2020,janner_when_2021} to capture such dynamics. However, these models are usually densely-connected neural networks and cannot be used for the purpose of explanation.

Psychological research suggests that people explain the world through causality \cite{sloman_causal_2005}. In this paper, we propose a novel framework that uses an interpretable world model to generate explanations. Rather than using a dense and fully-connected model, we perform causal discovery to construct a sparse model that is aware of the causal relationships within the dynamics of environments. In order to explain agents' decisions, the proposed causal model allows us to construct causal chains which present the variables causally influenced by the agent's actions. The proposed model advances the existing work that uses causality for explainable RL \cite{madumal_explainable_2020,madumal_distal_2020}, as it does not require a causal structure provided by domain experts, and is applicable to continuous action space.

Apart from interpreting the world, humans also use causality to guide their learning process \cite{cohen_rational_2020}. However, the trade-off between interpretability and performance \cite{gunning_xaiexplainable_2019,holzinger_explainable_2020,puiutta_explainable_2020} indicates that explainable models are usually inaccurate and can hardly benefit learning. On the contrary, our model is sufficiently accurate, leading to a performance close to dense models in MBRL. Therefore, we can train the agent and explain its decisions through exactly the same model, making explanations more faithful to the agent's intention. This is significant for overcoming the issue that post-hoc explanations like saliency maps can sometimes fail to faithfully unravel the decision-making process \cite{atrey_exploratory_2020}. 

Our main contributions are as follows: 1) We learn a causal model that captures the environmental dynamics without prior knowledge of the causal structure. 2) We design a novel approach to effectively extract the causal influence of actions, allowing us to derive causal chains for explaining the agent's decisions. 3) We show that our explanatory model is accurate enough to guide policy learning in MBRL.

\section{Background}
\subsection{Causality in Reinforcement Learning}

RL integrated with causality has recently become noticed by RL researchers. For example, Lu \shortcite{lu_deconfounding_2018}, Wang \shortcite{wang_provably_2021} and Yang \shortcite{yang_training_2022} \etal use causal inference to improve the robustness against confounders or adversarial intervention; Dietterich \etal \shortcite{dietterich_discovering_2018} improve learning efficiency by removing the variables unrelated to the agent's action; Nair \etal \shortcite{nair_causal_2019} construct a causal policy model; Seitzer \etal \shortcite{seitzer_causal_2021} improve exploration by detecting the causal influence of actions; Wang \shortcite{wang_causal_2022} and Ding \shortcite{ding_generalizing_2022} \etal investigate causal world models as we do. However, they mainly use causality to improve generalization rather than generate explanations.

Only a few studies consider improving explainability using a causal model. Madumal \etal \shortcite{madumal_explainable_2020} propose the Action Influence Model (AIM), a causal model specialized for RL, to generate explanations about the agent's actions. Another study \cite{madumal_distal_2020} further combines the AIM with decision trees to improve the quality of explanations. However, these approaches require finite action space, and the causal structure is given beforehand by human experts. In addition, they only consider a low-accuracy model, which cannot be used for policy learning. Volodin~\shortcite{volodin_causeoccam_2021} proposes a method to learn a sparse causal graph of hidden variables abstracted from high-dimensional observations, which improves the understanding of the environment. However, the approach provides no insight into the agent's behavior.

\subsection{Structural Causal Model}

A \textit{Structrual Causal Model} (SCM) \cite{pearl_causal_2016}, denoted as a tuple $(\mathcal{U}, \mathcal{V}, \mathcal{F})$, formalizes the causal relationships between multiple variables. $\mathcal{U}$ and $\mathcal{V}$ contain the variables of the SCM, where $\mathcal{U} = \{u_1, ..., u_p\}$ is the set of \textit{exogenous variables} and $\mathcal{V} = \{v_1, ..., v_q\}$ is the set of \textit{endogenous variables}. $\mathcal{F} = \{f^1, ..., f^q\}$ is the set of \textit{structural equations}, where $f^j$ formulates how $v_j$ is quantitatively decided by other variables. We call $f^j$ a ``structural" equation as it defines the subset of variables denoted as $PA(v_j) \subseteq \mathcal{U} \cup \mathcal{V} \setminus \{v_j\}$ (i.e., the parent variables) that directly decide the value of $v_j$.

An SCM is usually represented as a directed acyclic graph (DAG) $\mathcal{G} = (\mathcal{U}\cup\mathcal{V}, \mathcal{E})$ called the \textit{causal graph}. The node set of $\mathcal{G}$ is exactly $\mathcal{U}\cup\mathcal{V}$ and the edge set $\mathcal{E}$ is given by the structural equations: $(x_i, v_j) \in \mathcal{E} \Leftrightarrow x_i \in PA(v_j)$, where $x_i\in \mathcal{U}\cup \mathcal{V}$. For simplicity, symbols like $u_i$, $v_j$, and $x_i$ may denote the names or the values of the variables according to the context. In addition, we consider stochastic structural equations, where $f^j$ outputs the posterior distribution of $v_j$ conditioned on its direct parents $PA(v_j)$:
\begin{equation}
\mathrm{Pr}(v_j|PA(v_j)) \sim f^j(PA(v_j)).
\end{equation}

\subsection{Action Influence Model}

An AIM, denoted as the tuple $(\mathcal{U}, \mathcal{V}, \mathcal{F}, A)$, is a causal model specialized for RL. Here, $\mathcal{U}$ and $\mathcal{V}$ follow the same definition in SCM. $\mathcal{F}$ is the set of structural equations, and $A$ is the action space of the agent. Different from SCM, each structural equation is related to not only an endogenous variable but also a unique action in $A$. In other words, there exists a structural equation $f^j_{\bm{a}} \in \mathcal{F}$ for any action $\bm{a} \in A$ and any endogenous variable $v_j \in V$ to describe how $v_j$ is causally determined under action $\bm{a}$. We use $PA_{\bm{a}}(v_j) \subseteq \mathcal{U} \cup \mathcal{V} \setminus \{v_j\}$ to denote the causal parents of $v_j$ under action $\bm{a}$. Then, the posterior distribution of $v_j$ under action $\bm{a}$ is given by
\begin{equation}
\mathrm{Pr}(v_j|PA_{\bm{a}}(v_j), \bm{a}) \sim f^j_{\bm{a}}(PA_{\bm{a}}(v_j)).
\end{equation}
As a result, there exist overall $|A|\times|\mathcal{V}|$ causal equations in the AIM. In fact, we may also reckon the AIM as an ensemble of $|A|$ SCMs, where each SCM accounts for the influence of a unique action in $A$.

\section{Factorized MDP}
We are interested in tasks where the action and state can be factorized into multiple variables, and formalize such a task as a Factorized MDP (FMDP) denoted by the tuple $\left< S,A,O,R,P,T,\gamma\right>$. Here, $S$, $A$, and $O$ respectively denote the \textit{state space}, \textit{action space}, and \textit{outcome space}. Each state $\bm{s}\in S$ is factorized into $n_s$ state variables such that $\bm{s} = (s_1, ..., s_{n_s})$, where $s_i$ is the $i$-th state variable. Similarly, we have $\bm{a} = (a_1, ..., a_{n_a})$ for each action and $\bm{o}=(o_1,...,o_{n_o})$ for each outcome. Figure \ref{fig:causal_model_illustration}(a) illustrates an example of the factorization for a simple 2-grid environment called the Vacuum world \cite{russell_artificial_2010}. Its state variables include the $position$ of the vacuum and whether the places are clean ($clean_1$ and $clean_2$); it contains only one action variable $a$ that is chosen from Left, Right, and Suck (making the place clean). The action of Left or Right leads to an outcome of failure when blocked by the world boundary.

On each step, the agent observes the current state $\bm{s}$ and takes an action $\bm{a}$, then the state transits and the outcome is produced according to the transition probability $P(\bm{o',s'}|\bm{s,a})$, leading to a \textit{transition tuple} denoted as $\bm{\delta} = (\bm{s,a,o,s'})$. Meanwhile, the reward is given by the overall reward function $R(\bm{\delta})$. Following the reward decomposition \cite{juozapaitis_explainable_2019}, we factorize $R$ as the summation of $n_r$ \textit{reward variables}, given by $R(\bm{\delta}) = \sum_{i=1}^{n_r} r_i(\bm{\delta})$. $\gamma$ is the discount factor for computing returns. $T$ is the termination condition deciding whether the episode terminates based on the transition $\bm{\delta}$. How reward variables $\{r_i\}_{i=1}^{n_r}$ and the termination condition $T$ depend on the transition $\bm{\delta}$ is defined by users according to their demands. In those cases where $R$ and $T$ contain components unknown to users, we may put these components into the outcome variables. That is, we use an outcome variable $o_i$ to indicate the unknown reward, and define the corresponding reward variable as $r_i \equiv o_i$.

\section{The Proposed Framework}
\label{sec:method}

\begin{figure*}[tb]
    \centering
    \includegraphics[width=14cm]{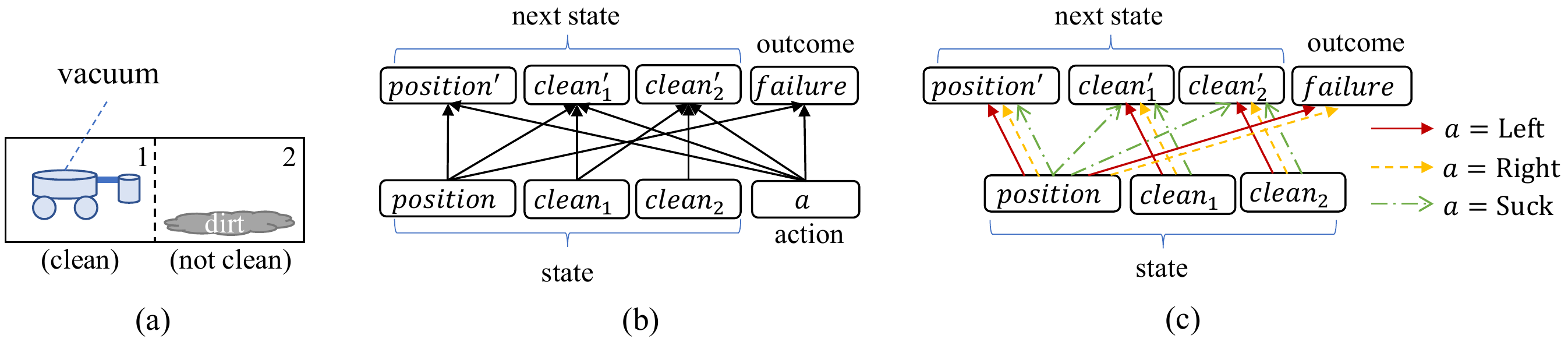}
    \caption{The illustration of causal models in the Vacuum world. (a) illustrates the Vacuum world, where $position=1$, $clean_1=True$, and $clean_2=False$. (b) and (c) respectively illustrate the causal graphs of the SCM and the AIM of the Vacuum world.}
    \label{fig:causal_model_illustration}
\end{figure*}

\begin{figure*}[tb]
	\centering
	\includegraphics[width=12.5cm]{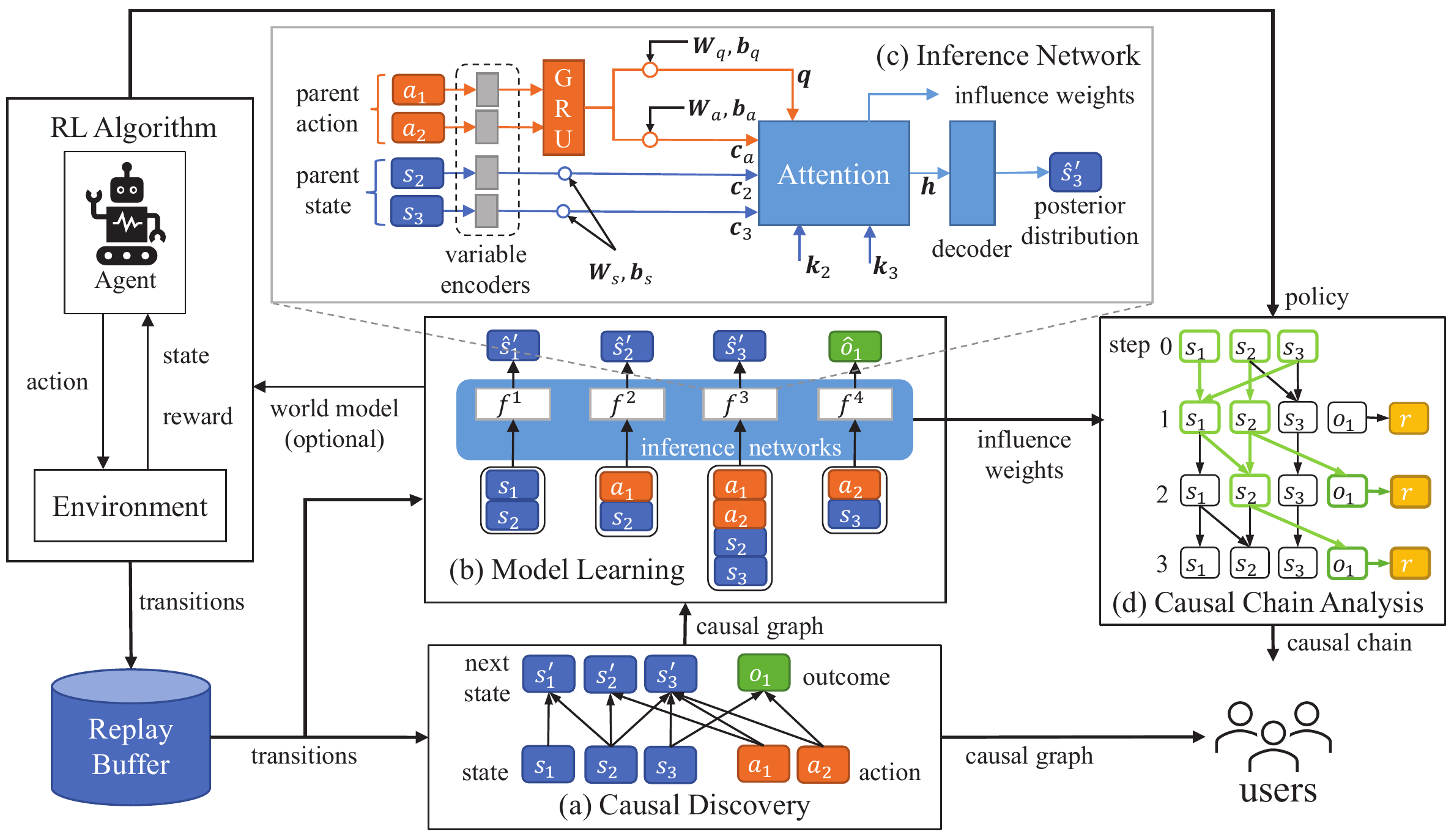}
	\caption{The illustration of the proposed framework. (a) shows an example of the causal graph identified by causal discovery. (b) illustrates the structure of the proposed model. (c) shows the inference network that approximates the structural equation of $s_3'$. (d) illustrates the causal chain analysis, where the causal chain is highlighted in bold and green.}
    \label{fig:framework}
\end{figure*}

There exist two perspectives on the causal model of the dynamics of the environment: 1) In a unified SCM, action variables are merely nodes of the causal structure and are treated evenly as state variables. 2) In an AIM, each action specifies a peculiar causal structure, leading to a good understanding of both the environment and the agent's decisions \cite{madumal_explainable_2020}. To make it clearer, Figure \ref{fig:causal_model_illustration} illustrates how a unified SCM and an AIM pertain to our setting in the above-mentioned Vacuum world. However, directly learning an AIM is intractable as we must divide data into $|A|$ subsets to respectively learn $|A|$ SCMs. This reduces sample efficiency, produces redundant parameters, and cannot be applied to infinite action space. Therefore, we seek to build a unified SCM that can be converted to an AIM based on specially-designed structural equations. As illustrated in Figure~\ref{fig:framework}(a), the exogenous variables in this SCM are the state variables $\bm{s}$ and action variables $\bm{a}$; the endogenous variables are the next-state variables $\bm{s}'$ and outcome variables $\bm{o}$. In the following description, we define $\bm{u} := (\bm{s},\bm{a}) = (s_1, ..., s_{n_s}, a_1, ..., a_{n_a})$ as the input (exogenous) variables and $\bm{v} := (\bm{s}',\bm{o}) = (s_1', ..., s_{n_s}', o_1,...,o_{n_o})$ as the output (endogenous) variables of our model. In this way, a transition tuple can also be written as $\bm{\delta} = (\bm{u},\bm{v})$.

The workflow of the proposed framework is illustrated in Figure~\ref{fig:framework}. When the agent interacts with the environment, we store the transition data into a replay buffer $\mathcal{D}$. Using the stored transitions, we perform \textit{causal discovery} to identify the causal graph of the above-mentioned SCM. Then, we fit the causal equations for the variables of the next state $\bm{s'}$ and the outcome $\bm{o}$ using \textit{Inference Networks}. Together with the known reward function $R$ and termination condition $T$, we construct a causal world model that captures the dynamics of the environment. The attention weights (influence weights) in the inference networks capture the action influence, which allows us to perform \textit{causal chain analysis} to reveal the variables that are causally influenced by the agent's action. The discovered causal graph interprets the environmental dynamics, and the causal chain analysis provides explanations about the agent's decisions. Moreover, this causal world model can be used by MBRL algorithms to facilitate learning.

\subsection{Causal Discovery}
\label{sec:causal_discovery}

We assume that output variables $\bm{v}$ are produced independently conditioned on the input variables $\bm{u}$, as independence underlies the intuition of humans to segment the world into components. Under this assumption, it is proven that the causal graph is bipartite, where no lateral edge exists in $\bm{v}$, and each edge starts in $\bm{u}$ and ends in $\bm{v}$. Therefore, we only need to determine whether there exists a causal edge for each variable pair $(u_i, v_j)$, where $1 \leq i \leq n_s + n_a$ and $1 \leq j \leq n_s+n_o$. Studies have shown that \textit{conditional independent tests} (CITs) can be used to perform efficient causal discovery \cite{wang_causal_2022,ding_generalizing_2022}. In this work, we implement CITs through Fast CIT \cite{chalupka_fast_2018} and determine each edge using the following rule:
\begin{equation} \label{eq:causal_discovery_fmdp}
    u_i \in PA(v_j) \Longleftrightarrow (u_i \nupmodels v_j | \bm{u}_{-i}),
\end{equation}
where $\bm{u}_{-i}$ denotes all variables in $\bm{u}$ other than $u_i$. In Appendix \ref{appendix:causal_discovery}, we provide the theoretical basics of causal discovery and prove a theorem showing that Equation \ref{eq:causal_discovery_fmdp} leads to sound causal graphs.

\subsection{Attention-based Inference Networks}
\label{sec:inference_networks}
To perform causal inference on the discovered causal graph, we fit the structural equation of each output variable $v_j$ using an inference network denoted as $f^j$, which takes the causal parents $PA(v_j)$ as inputs and predicts the posterior distribution $\textrm{Pr}(v_j|PA(v_j))$. These inference networks should adapt to the structural changes of the causal graph, as the agent's exploratory behaviors may reveal undiscovered causal relationships and lead to new causal structures. To achieve this, Ding \etal \shortcite{ding_generalizing_2022} use Gated Recurrent Unit (GRU) networks that sequentially input all parent variables without discriminating the state and the action. To model the action influence, we design the attention-based inference networks as illustrated in Figure \ref{fig:framework}(c).

To handle heterogenous input variables (which may be scalars or vectors of different lengths), we first use \textit{variable encoders} (each uniquely belongs to an input variable) to individually map the input variables to vectors of the same length. These encoders are shared by the inference networks of all output variables. In particular, we use $\tilde{\bm{u}} = \{\tilde{s}_1, ...,\tilde{s}_n,\tilde{a}_1, ...,\tilde{a}_m\}$ to denote these encoding vectors of input variables.

Then, for each inference network $f^j$, we compute the \textit{contribution vectors} of the parent state variables through linear transforms:
\begin{equation}
    \bm{c}^j_{i} = \bm{W}_s^j \tilde{s}_i + \bm{b}_s^j,\quad s_i\in Pa(v_j).
\end{equation}
The usage of these contribution vectors is equal to the ``value vectors" in key-value attention. We use the term ``contribution vectors" since the word ``value" is ambiguous in the context of RL.

Each inference network $f^j$ contains a GRU network $g^j$, which receives the action variables in $PA(v_j)$ and outputs the action embedding $\bm{e}^j$. Then, we feed this action embedding into linear transforms to respectively obtain the query vector $\bm{q}^j$ and the action contribution vector $\bm{c}^j_a$:
\begin{gather}
    \bm{e}^j = GRU^j\left(
        \{\tilde{a}_i\}_{a_i\in PA(v_j)}
    \right), \\
    \bm{q}^j = \bm{W}^j_q \bm{e}^j + \bm{b}^j_q, \\
    \bm{c}^j_{a} = \bm{W}^j_a \bm{e}^j + \bm{b}^j_a .
\end{gather}
The projection matrices $\bm{W}^j_s, \bm{W}^j_q, \bm{W}^j_a$ and bias vectors $\bm{b}^j_s, \bm{b}^j_q, \bm{b}^j_a$ are all trainable parameters of $f^j$. We use the superscript $j$ to indicate that these parameters belong to $f^j$.

Each state variable $s_i$ is allocated a key vector $\bm{k}_i$, which is a trainable parameter learned by gradient descent. We do not use the superscript $j$ for key vectors as they are shared by inference networks of all output variables. The \textit{influence weights} (i.e., attention weights) of the state variables in $PA(v_j)$ and the action are then computed by
\begin{equation}
    \alpha_{i}^j = \frac{\exp(\bm{k}_{i}^T \bm{q}^j)}
    {1 + \sum_{s_{i'}\in PA(v_j)} \exp(\bm{k}_{i'}^T \bm{q}^j)},
\end{equation}
\begin{equation}
    \alpha_{a}^j = \frac{1}
    {1 + \sum_{s_{i'}\in PA(v_j)} \exp(\bm{k}_{i'}^T \bm{q}^j)}.
\end{equation}
We then compute the hidden representation of the posterior distribution using the weighted sum of the value vectors:
\begin{equation}
    \bm{h}^j = \sum_{s_i\in PA(v_j)} \alpha_{i}^j \cdot \bm{c}^j_{i} +
    \alpha_{a}^j \cdot \bm{c}^j_a .
\end{equation}
Finally, the \textit{distribution decoder} $D^j$ maps $\bm{h}^j$ to the predicted posterior distribution:
\begin{equation}
    \textrm{Pr}(v_j|PA(v_j)) \sim D^j(\bm{h}^j).
\end{equation}
We assume the type of this posterior distribution is previously known and $D^j$ only outputs the parameters of the distribution. In our implementation, we use normal distribution (parameterized by the mean and variance) for real-number variables and use categorical distribution (parameterized by the probability of each class) for discrete variables.

The inference networks $\{f_j\}_{j=1}^{n_s+n_o}$ are trained by maximizing the log-likelihood of the transition data stored in $\mathcal{D}$, written as
\begin{equation}
    \mathcal{L}_{infer} = \sum_{j=1}^{n_s+n_o} \frac{1}{|\mathcal{D}|} \sum_{\bm \delta \in \mathcal{D}} \log \textrm{Pr}(v_j|PA(v_j)).
    \label{eq:loss}
\end{equation} 

\subsection{Causal Chain Analysis}
\label{sec:generating_explanations}
In order to generate explanations, we first describe how our model can be converted to an AIM. Noticing that the key vectors $\{\bm{k}_i\}_{i=1}^n$ are trainable parameters, the influence weights only depend on the numeric value of action variables. Therefore, the influence weight $\alpha^j_i$ captures how much the output variable $v_j$ depends on state variable $s_i$ under the given action $\bm{a}$. In order to generate laconic explanations, we define $PA_{\bm{a}}(v_j) := \{s_i\in PA(v_j)\ |\ \alpha_i^j > \tau \}$ as the parent set of $v_j$ with salient dependencies under the action $\bm{a}$, where $\tau \in [0,1] $ is a given threshold. In this way, we convert the SCM to an AIM, where the structural equation for $v_j$ under action $\bm{a}$ is written as
\begin{equation}
    f^j_{\bm{a}}(PA_{\bm{a}}(v_j)) =
    D^j (\sum_{s_i\in PA_{\bm{a}}(v_j)} \alpha_{i}^j \cdot \bm{c}^j_{i} (s_i) +
    \alpha_{a}^j \cdot \bm{c}^j_a ).
    \label{eq:aim}
\end{equation}

Since we use the AIM for the purpose of explanation, it is tolerable to set a larger threshold, which allows us to ignore parent variables that are not influential enough. Madumal \etal have introduced methods to generate good explanations using an AIM. The key is to build a causal chain containing the variables that (i) are causally affected by the actions, and (ii) causally lead to rewards. A single causal chain starting from state $\bm{s}$ and action $\bm{a}$ leads to the explanation for ``why the agent took $\bm{a}$ on $\bm{s}$". Contrastive explanations for ``why the agent took $\bm{a}$ instead of $\bm{b}$ on $\bm{s}$" can be obtained by comparing the causal chines produced by the factual action $\bm{a}$ and the counterfactual action $\bm{b}$. Details can be found in Appendix \ref{appendix:explain} and the AIM paper \cite{madumal_distal_2020}.

The rest of this section introduces how to derive a causal chain starting from the state $\bm{s}^t$ at step $t$ and an action $\bm{a}^t$ (can be factual or counterfactual) using our model. First of all, we use our model and the agent's policy to simulate the most-likely trajectory $\bm{\delta}^t, \bm{\delta}^{t+1}, ..., \bm{\delta}^{t+H-1}$, where $H$ denotes the number of simulation steps. The symbol $\bm{\delta}^{t+k}$ denotes the transition tuple on step $t+k$, where the actions $\bm{a}^{t+k}$ for $k\geq 1$ are produced by the agent's policy. For a factual causal chain, this simulation is not necessary if factual data of these future states and actions is available.

Then, we build an extended graph containing the state, outcome, and reward variables of these $H$ steps. The edges of this graph accord to the structure of the AIM derived above. That is, if $s_i \in PA_{\bm{a}^{t+k}}(v_j)$ then there exists an edge from $s_{i}^{t+k}$ to $v_{j}^{t+k}$ for all $k = 1, ..., H$. It is worth mentioning that we treat the first transition $\bm{\delta}^t$ differently since $\bm{a}^t$ is exactly the action being explained: If $PA(v_j)\cap \bm{a}^t = \emptyset$, then $v_j^t$ is not affected by the choice of $\bm{a}^t$. In this case, no edge will be established from any state $s_i^t$ to $\bm{v}_j^t$.

Afterward, the explainee may specify the target variables (a subset of reward variables) he/she is interested in. Otherwise, all reward variables will be considered. We perform a graph search from the starting state variables $\bm{s}^t$ and highlight all paths from $\bm{s}^t$ to the target rewards. These paths together form the \textit{causal chain} of action $\bm{a}^{t}$ starting from $\bm{s}^t$. Based on this causal chain, the explanation can be presented as a picture or a natural-language description.

\subsection{Model-based RL}
\label{sec:mbrl}

XAI literature has widely discussed the trade-off between interpretability and performance, which is also reflected in our model. A sparser causal graph (discovered using a smaller threshold $\eta$) is usually easier to read and produces clearer explanations. However, it also enforces the model to infer posterior distributions using less information from input variables $\bm{u}$, leading to inferior accuracy. In Appendix \ref{appendix:proof_tradeoff}, we provide a theorem that formally shows that decreasing the threshold $\eta$ leads to a denser causal graph (i.e., lower interpretability) and also higher predicting accuracy.

In order to show our model is accurate enough to do more than generate post-hoc explanations, we consider applying our world model to MBRL to facilitate policy learning. We use a bootstrap ensemble containing 5 models to alleviate the effect of the epistemic uncertainty \cite{chua_deep_2018}. For each iteration, we first collect real transition data into the model buffer $\mathcal{D}$. Then, we update the world model by causal discovery and fitting structural equations using the data in $\mathcal{D}$. Afterward, we perform $k$-step model-rollouts \cite{janner_when_2021} to generate simulated data for updating the policy.  In our implementation, the policy is trained using Proximal Policy Optimization \cite{schulman_proximal_2017}. The pseudo-code of the learning procedure is given in Appendix \ref{appendix:alg_mbrl}.

\section{Experiments}
We present examples of causal chains in two representative environments: Lunarlander-Continuous for the continuous action space, and Build-Maine for the discrete action space. To verify whether our approach can produce correct causal chains, we design an environment to measure the accuracy of recovering causal dependencies of the ground-truth AIM. To evaluate the performance of our model in MBRL, we perform experiments in two extra environments: Cartpole and Lunarlander-Discrete. The Build-Marine environment is adapted from one of the StartCraftII mini-games in SC2LE \cite{samvelyan_starcraft_2019}; the Cartpole and Lunarlander environments are classic control problems provided by OpenAI Gym \cite{brockman_openai_2016}. Our source code is available at 
\href{https://github.com/EaseOnway/Explainable-Causal-Reinforcement-Learning}{https://github.com/EaseOnway/Explainable-Causal-Reinforcement-Learning}.

\subsection{Explanation Results}

\begin{figure}[tb]
    \centering
    \subfigure[Lunarlander-Continuous]{
        \includegraphics[height=2.8cm]{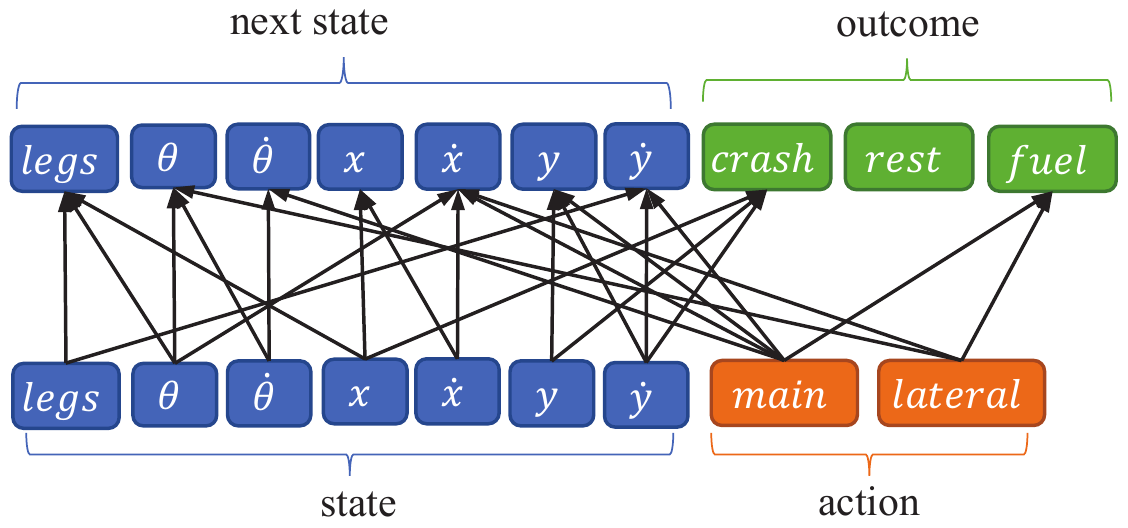}
        \label{fig:lunarlander_graph}
    }

    \centering
    \subfigure[Build-Marine]{
        \includegraphics[height=2.8cm]{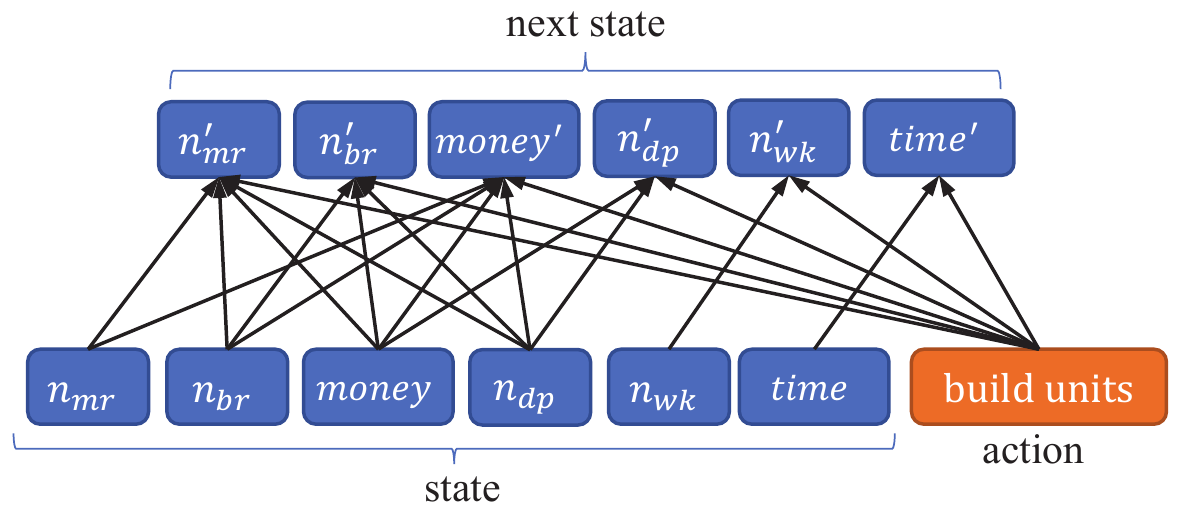}
        \label{fig:buildmarine_graph}
    }
    \caption{The discovered causal graphs of two environments.}
    \label{fig:explain_graphs}
\end{figure}

\begin{figure*}[tb]
    \centering
    \includegraphics[width=13cm]{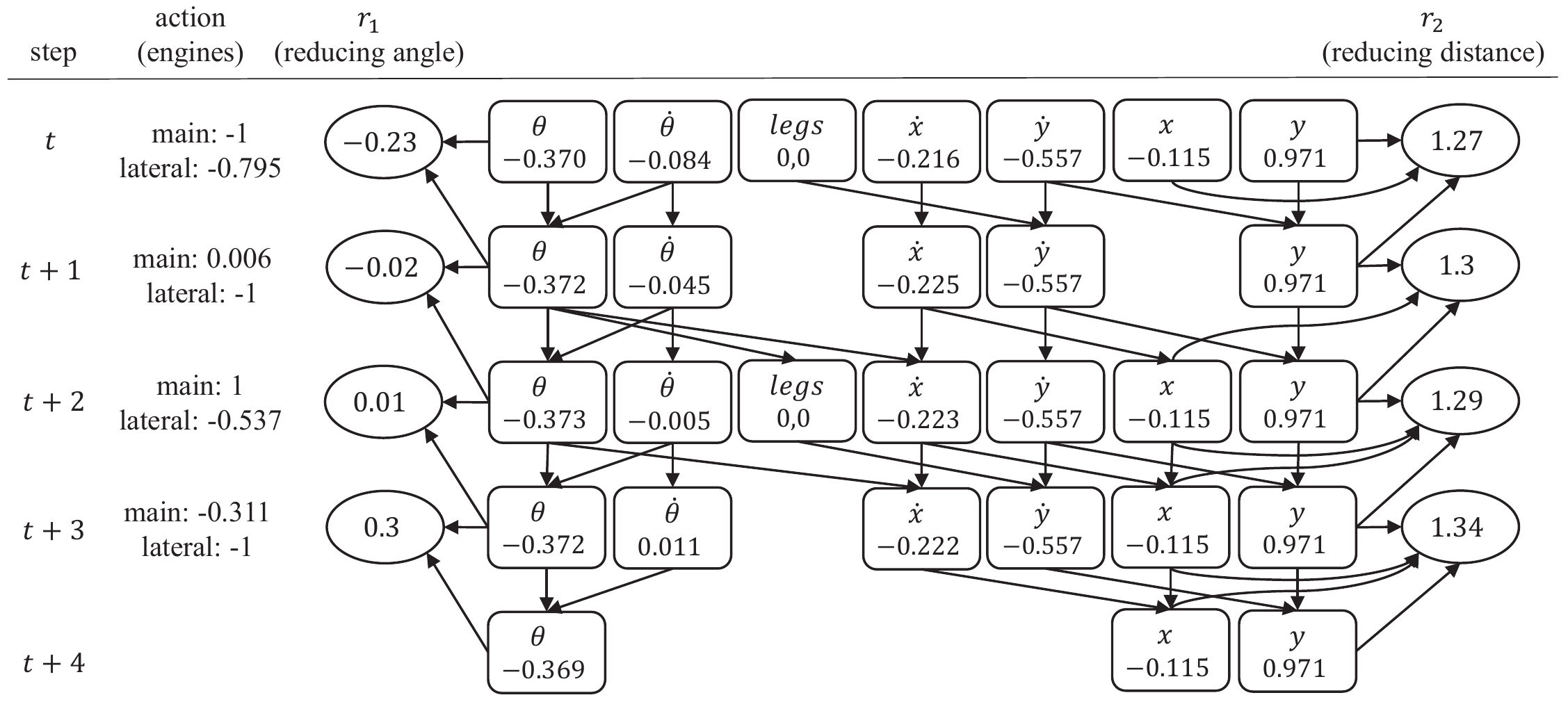}
    \caption{An example of a 4-step causal chain on Lunarlander-Continuous}
    \label{fig:lunarlander_chain}
\end{figure*}

\begin{figure}[tb]
    \centering
    \includegraphics[width=8cm]{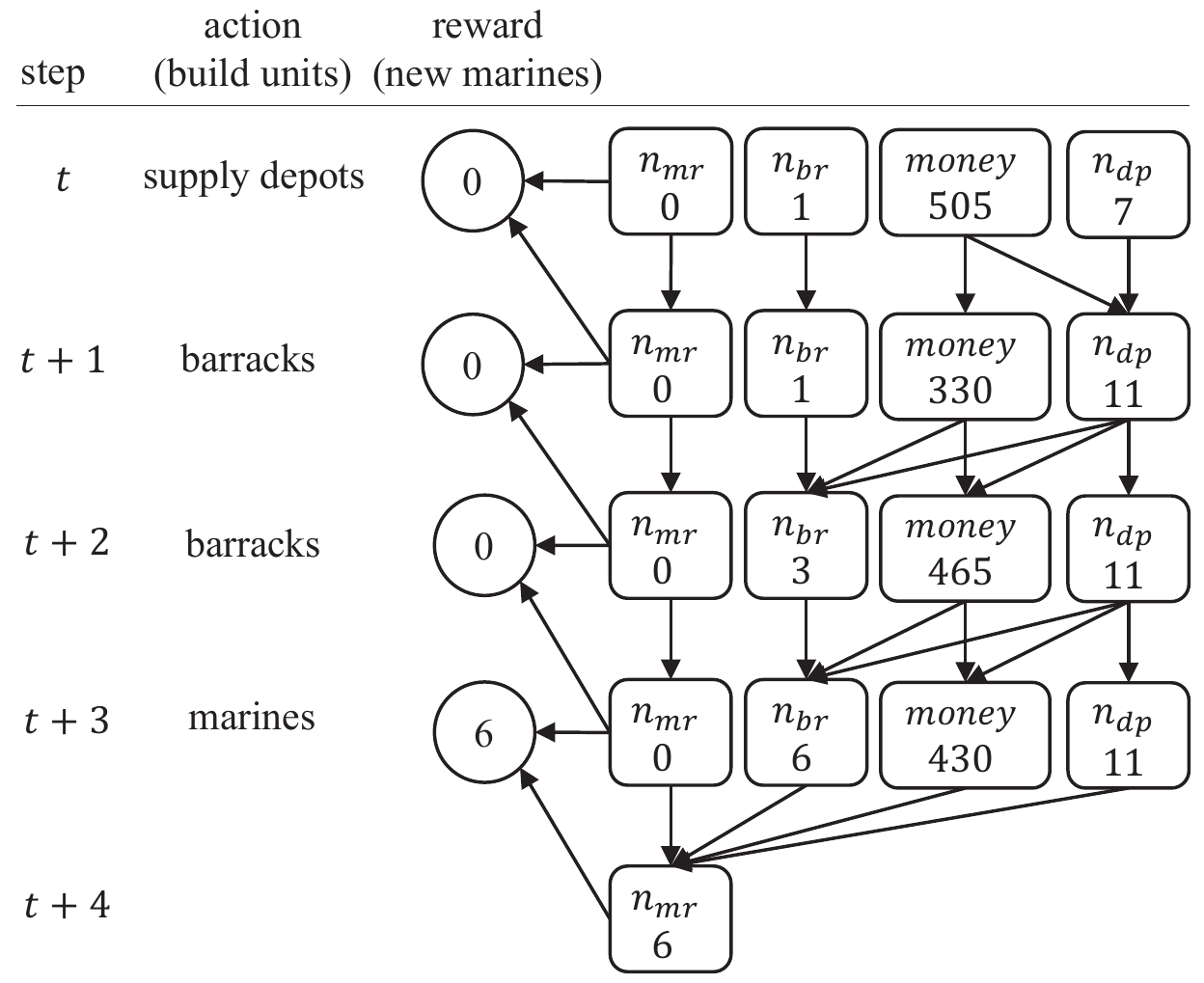}
    \caption{An example of a 4-step causal chain on Build-Marine}
    \label{fig:buildmarine_chain}
\end{figure}

\subsubsection{Lunarlander-Continuous}
We factorize the state into 7 variables $(x, y, \dot{x}, \dot{y}, \theta, \dot{\theta}, legs)$ indicating 1) the horizontal position, 2) the vertical position, 3) the horizontal velocity, 4) the vertical velocity, 5) the angle, 6) the angle velocity, and 7) whether the two legs are in contact with the ground. The action includes 2 continuous variables ranged in $(-1, 1)$, respectively controlling the throttles of the main and the lateral engines. The environment contains 3 outcome variables, including 1) the fuel cost due to firing the engine, 2) whether the lander crashed, and 3) whether the rocket is resting.

In this experiment, we learn a post-hoc model to generate explanations for a previously trained policy. We first use the policy (with noise) to collect 150k samples into the buffer $D$. Then, we use these samples to discover the causal graph (with the threshold $\eta=0.05$) and train the inference networks. The resulting causal graph is presented in Figure~\ref{fig:lunarlander_graph}. The environment contains many kinds of rewards, leading to complicated causal chains if we consider them all. To make our explanation clearer, we present a causal chain in Figure~\ref{fig:lunarlander_chain} considering only two kinds of rewards: 1) the reduction of the distance to the target location, and 2) the reduction of the angle (i.e., balancing the rocket). This causal chain shows that the agent's action $\bm{a}^t$ first influences the velocities ($\dot{\theta}$, $\dot{x}$, and $\dot{y}$) and thereby reduces (or increases) the angle ($\theta$) and the distance ($\sqrt{x^2 + y^2}$). In addition, we observe that no parent of the outcome variable $rest$ is discovered in the causal graph. This means the policy provides insufficient opportunities to reveal its causality. As a result, the variable $rest$ is excluded from all causal chains, showing that it is not an important consideration in the agent's policy.

\subsubsection{Build-Marine}
\label{sec:build_marine_result}

The original observation space provided by the SC2LE interface contains hundreds of variables, which is intractable for causal discovery. In our implementation, we define the state as the tuple containing only 6 variables denoted as $(n_{wk}, n_{mr}, n_{br}, n_{dp}, money, time)$, namely 1) the number of workers, 2) the number of marines, 3) the number of barracks, 4) the number of supply depots, 5) the amount of money, and 6) the game time. We are interested in the macro-level decision-making and therefore define the action as one discrete variable indicating which unit (workers, marines, barracks, supply depots, or none) to be built. The micro-level control of building these units (e.g., determining where to place the new barracks) is implemented by simple rules. The goal is to produce as many marines as possible within 15 minutes. Therefore, the player is rewarded with $1$ for every newly produced marine. In addition, this environment contains no outcome variable.

In this experiment, both the policy and the causal model are obtained by model-based learning (see Section \ref{sec:mbrl}). We present the final causal graph in Figure~\ref{fig:buildmarine_graph} (discovered using the threshold $\eta=0.15$) and an example of the causal chain in Figure~\ref{fig:buildmarine_chain}. The causal chain shows that our attention-based inference networks successfully reason the causal dependencies under different actions, which reflect the following rules of the StarCraftII game: 1) Building new barracks requires at least one supply depot; 2) marines are built from barracks; 3) the number of marines is limited by the number of supply depots; and 4) building more units requires sufficient money in hand. This causal chain explains the reason why the agent builds supply depots: to gain permission to build barracks and provide enough supplies for building marines. Interestingly, we discover no causal relationship between $n_{wk}$ and $money'$. For human players, it is common sense that more workers increase the efficiency of collecting minerals and thus lead to a higher income. Since the causal model is learned using the transition data produced along with policy training, this missing edge indicates that the agent explored inadequately for building more workers, providing insufficient evidence to reveal this causal relation.

\subsection{Accuracy of Recovering Action Influence}
\label{sec:accuracy_aim}

\begin{table}[]
    \centering
    \begin{tabular}{c|ccc}
        \hline
                     & Direct & Full+Attn & Caus+Attn (ours)  \\
        \hline
        AIM accuracy & 97.0 \% & 90.0\% & \textbf{99.3\%} \\
        \hline
    \end{tabular}
    \caption{The accuracy of recovering the causal dependencies of the AIM. ``Direct" means the direct approach mentioned in Section \ref{sec:accuracy_aim}; ``Full" means using a full graph; ``Caus" means using a causal graph; and ``Attn" means using attention.}
    \label{tab:aim_accuracy}
\end{table}

Good explanations are generated from correct causal chains, which require us to accurately recover the AIMs of environments. In Section \ref{sec:method}, we have mentioned a \textit{Direct} approach that learns the AIM by splitting the data buffer $\mathcal{D}$ into $|\mathcal{A}|$ sub-buffers and performing causal discovery for each action $a\in \mathcal{A}$. Though this direct approach is theoretically sound, it suffers from poor sample efficiency and high computational complexity. Noticing that the causal dependencies under different actions usually share similar structures, our approach takes 2 stages: 1) In the \textbf{causal stage}, we learn a unified SCM, whose causal graph summarizes causal dependencies for all actions; 2) in the \textbf{attention stage}, we then transfer this SCM into an AIM based on attention weights (influence weights).

To verify whether our approach can accurately recover the AIM, we design a simple environment that contains spurious correlations to confuse neural networks (see Appendix \ref{appendix:env_aimtest} for details). We compare our approach with two baselines: 1) the Direct approach mentioned above, and 2) a non-causal approach that uses a full causal graph and only relies on attention. The accuracy of recovering the ground-truth causal dependencies of the AIM using non-i.i.d data is shown in Table \ref{tab:aim_accuracy}. These results show that: 1) the causal graph discovered in the causal stage precludes most spurious correlations, making our approach more effective than the Direct approach in practice; 2)  and attention alone is insufficient to accurately extract the causal influence of actions.

Further, we examined the causal chains derived solely from attention (where full causal graphs are used). In these chains, we found plenty of spurious correlations, which lead to unreasonable explanations (e.g., ``the number of supply depots naturally grows with time" in Build-Marine). An example of such a causal chain and the related discussion are provided in Appendix \ref{appendix:causality_case_study}. This result shows that causal discovery is an indispensable process for producing reasonable explanations.

\subsection{Performance in Model-Based RL}

\begin{figure}[tb]
    \centering
    \includegraphics[width=8.5cm]{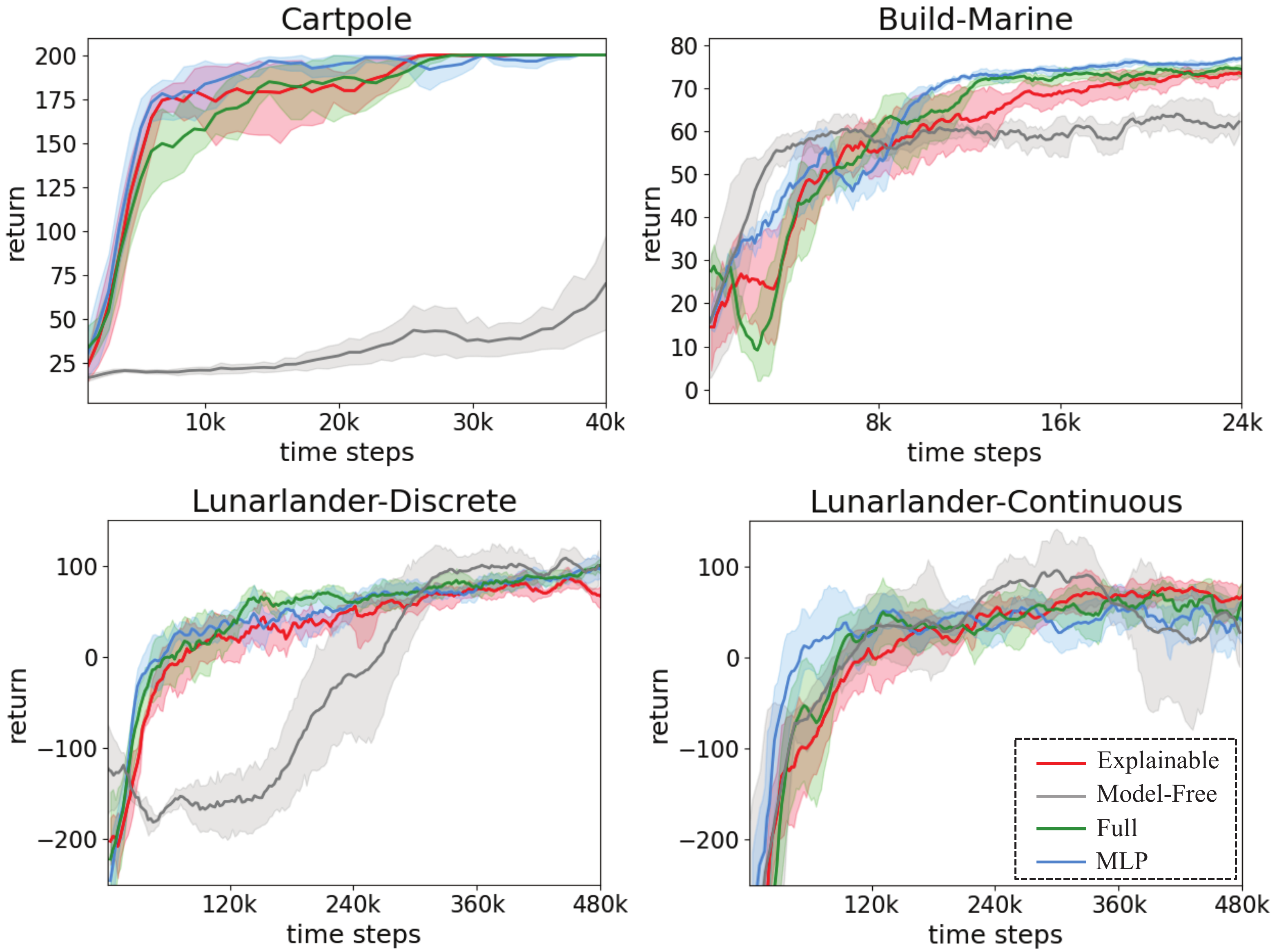}
    \caption{The training curves. Our explainable model (red) is compared to non-explainable dense models (green and blue) to show the performance cost of using a sparse causal graph. The grey curves show the performance without models.}
    \label{fig:mbrl_curve}
\end{figure}

We evaluate the performance of model-based policy learning using our explanatory model.  We compare the learning performance with the \textbf{Model-Free} approach that learns policies without models. In addition, we consider two dense models as baselines: 1) the model that concatenates all exogenous variables $\bm{u}$ as inputs and infers endogenous variables using a Multi-Layer Perceptron (\textbf{MLP}), and 2) the \textbf{Full} model that adopts the same networks as ours whereas uses a full causal graph.

The MLP model is the most commonly used in MBRL, and the Full model follows the state-of-the-art modular architecture \cite{ke_systematic_2021}. These dense models are not suitable for generating explanations. However, they are more accurate (if well-trained) than our explainable model as they are allowed to predict each output variable based on the complete inputs. Existing studies show that causal models generalize better than dense models \cite{wang_causal_2022,ding_generalizing_2022}. However, we focus on ordinary learning problems and do not consider using our model for generalization. Therefore, we stress that the goal of this experiment is not to obtain a higher performance than dense baselines. On the contrary, We aim to figure out: 1) whether the proposed model can be of help to the learning process, and 2) how much performance our sparse model sacrifices for improved explainability. 
 
The learning curves are shown in Figure \ref{fig:mbrl_curve}. In all environments, the performance of our explanatory model is very close to the dense baselines. Compared to the model-free approach, the model-based approaches significantly learn faster in Cartpole and Lunarlander-Discrete and converge to higher returns in Build-Marine. The returns of all tested approaches are close in Lunarlander-Continuous, whereas the model-based approaches improve the stability of learning. These results show that our model improves explainability at an acceptable cost in performance and well balances the interpretability-accuracy trade-off. Therefore, our model can simultaneously guide policy learning and explain decisions, leading to better consistency between explanations and the agent's cognition of the environment.

\section{Conclusion and Future Work}
This paper proposes a framework that learns a causal world model to generate explanations about agents' actions. To achieve this, we perform causal discovery to identify the causal structure in the environment and fit causal equations through attention-based inference networks. These inference networks produce the influence weights that capture the influence of actions, which allow us to perform causal chain analysis in order to generate explanations. The proposed framework does not require the structural knowledge provided by human experts and is applicable to infinite action spaces. Apart from generating explanations, we successfully applied our model to model-based RL, showing that the model can be the bridge between learning and explainability.

A weakness of our approach is that it requires a known factorization of the environment, which limited its application scope. There exists a number of studies aiming to learn the causal feature set from raw observations \cite{zhang_invariant_2020,volodin_causeoccam_2021,zhang_learning_invariant_2021}. Future work will put representation learning into consideration for better applicability. In addition, we currently consider model-based policy learning as the usage of our model apart from generating explanations. However, this usage does not make full use of the advantage of a causal model. Future work will investigate better usage of our model to further improve learning.

\section*{Acknowledgments}

This work was supported in part by National Key R$\&$D Program of China (No.2022ZD0116405) and in part by the National Nature Science Foundation of China under Grant (62073324).

\bibliographystyle{named}
\bibliography{references}


\appendix
\section{Causal Discovery}
\label{appendix:causal_discovery}

In this section, we introduce some basics of causal discovery and then prove the soundness of our approach for causal discovery. First of all, we introduce the concept of Markov Compatibility \cite{pearl_causality_2000}, which describes whether a directed acyclic graph (DAG) is able to represent the dependencies of a group of random variables. 

\begin{definition}[Markov Compatibility] \label{def:markov_compatibility} Assume $\bm{x}=(x_1,x_2,\cdots,x_n)$ is a group (ordered set) of random variables and $\Pr$ is a probability function on $\bm{x}$. Assume $\mathcal{G}$ is a DAG whose nodes are these variables, where the parent set of $x_i$ is denoted as $PA(x_i)$. If we have
\begin{equation}
    \Pr(\bm{x}) = \prod_{i=1}^n \Pr(x_i|PA(x_i)),
\end{equation}
then we say that $\Pr$ and $\mathcal{G}$ are \textit{compatible}, or that $\mathcal{G}$ \textit{represents} $\Pr$.
\end{definition}

The goal of causal discovery is to find some DAG $\mathcal{G}$ to represent a given probability $\Pr$. Now, we introduce two important concepts for causal discovery: the d-separation and causal faithfulness.

\begin{definition}[d-separation] \label{def:d-separation} Assume $\mathcal{G}$ is a DAG on a set of variables, where $\bm{x}$, $\bm{y}$, and $\bm{z}$ are disjoint subsets of variables. We say that $\bm{x}$ and $\bm{y}$ is \textit{d-separated} by on $Z$ (denoted as $\bm{x} \upmodels_{\mathcal{G}} \bm{y}|\bm{z}$), if every undirected path $\bm{p}$ from $\bm{x}$ to $\bm{y}$ satisfies:
\begin{enumerate}
    \item There exists a forward chain $a \rightarrow b \rightarrow c$, a backward chain $a \leftarrow b \leftarrow c$, or a fork $a \leftarrow b \rightarrow c$ in $p$ such that $b\in Z$.
    \item For every collision structure $a \rightarrow b \leftarrow c$ in $p$, $Z$ does not contain $b$ or any descendant of $b$.
\end{enumerate}
\end{definition}

\begin{theorem}[d-separation Criterion] \label{theorem:d-separation} Assume $\mathcal{G}$ is a DAG of a set of variables. Assume $\bm{x}$, $\bm{y}$, and $\bm{z}$ are disjoint subsets of variables. The following propositions hold:
\begin{enumerate}
    \item (Global Markov Property \cite{peters_elements_2017}) Assuming $\Pr$ is a probability function such that $\Pr$ and $\mathcal{G}$ are compatible, then
    \begin{equation}
    (\bm{x} \upmodels_{\mathcal{G}} \bm{y}|\bm{z}) 
    \Rightarrow
    (\bm{x} \upmodels_{\Pr} \bm{y}|\bm{z}),
    \end{equation}
    where $\upmodels_{\Pr}$ means conditional independence under $\Pr$.
    \item  If $(\bm{x} \upmodels_{\Pr} \bm{y}|\bm{z})$ holds for every  probability function $\Pr$ that is compatible with $\mathcal{G}$, we have $(\bm{x} \upmodels_{\mathcal{G}} \bm{y}|\bm{z})$ \cite{pearl_causality_2000}.
\end{enumerate}

\end{theorem}

\begin{definition}[Causal Faithfulness] \label{def:faithfulness} Assume $\mathcal{G}$ and $\Pr$ are respectively a DAG and a probability function on a set of variables. We say that $\Pr$ is \textit{faithful} to $\mathcal{G}$, if 
\begin{equation}
    (\bm{x} \upmodels_{\Pr} \bm{y}|\bm{z})
    \Rightarrow
    (\bm{x} \upmodels_{\mathcal{G}} \bm{y}|\bm{z}),
\end{equation}
for all disjoint subsets $\bm{x}$, $\bm{y}$, and $\bm{z}$ of variables.
\end{definition}

Causal faithfulness indicates that all conditional independent relationships are due to the structure of the DAG instead of rare coincidence. In fact, studies have shown that if $\Pr$ is compatible with $\mathcal{G}$, the chance of $\Pr$ to be not faithful to $\mathcal{G}$ is extremely low \cite{spirtes_causation_2001}. In causal discovery, it is usually assumed that the probability $\Pr$ is faithful to the DAG $\mathcal{G}$ that we are looking for, which makes the structure of $\mathcal{G}$ recognizable. 

We have not distinguished exogenous and endogenous variables above. The \textit{exogenous variables} are considered the inputs of a system, and thus their causality does not need to be discussed. In other words, our causal graph only describes the causality of \textit{endogenous variables}, whereas the causality of exogenous variables is ignored. Unless otherwise specified, the letter $\bm{v}$ denotes the set of endogenous variables and the letter $\bm{u}$ denotes the set of endogenous variables in the following discussion. Given a probability function $\Pr$ of some variables $\bm{x} = (\bm{u},\bm{v})$, the goal of causal discovery now becomes determining only the causal parents of endogenous variables $\bm{v}$ in a DAG $\mathcal{G}$ that represents $\Pr$.

It is worth mentioning that, in the definition given by Pearl \shortcite{pearl_causal_2016}, it is assumed that exogenous variables are independent of each other. However, this requirement is released in our definition since the current state and action variables are usually correlated due to the dependencies underlying the agent's policy and the history transitions. Ignoring these correlations leads to spurious edges, which are detrimental to generating reasonable explanations.  For example, in Figure \ref{fig:spurious_edges:policy}, we show that a backdoor path occurs when the actions are sampled by a policy dependent on the current state; in Figure \ref{fig:spurious_edges:history} we show that a backdoor path occurs considering the transition history even if the actions are sampled randomly. In both cases, $s_1$ and $s_2'$ are correlated, making $s_1$ an ``spurious parent" of $s_2'$ in the discovered causal graph.

\begin{figure}[h]
    \centering
    \subfigure[]{
      \begin{minipage}[t]{0.4\linewidth}
        \centering
        \includegraphics[width=3cm]{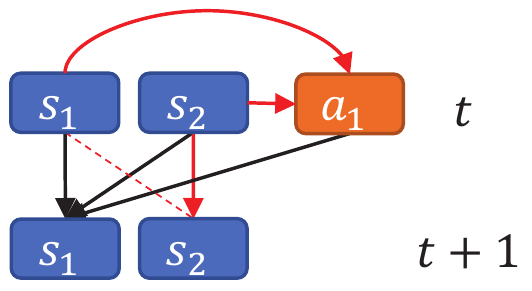}
     \end{minipage}
      \label{fig:spurious_edges:policy}
    }
    \subfigure[]{
        \begin{minipage}[t]{0.4\linewidth}
            \centering
            \includegraphics[width=3cm]{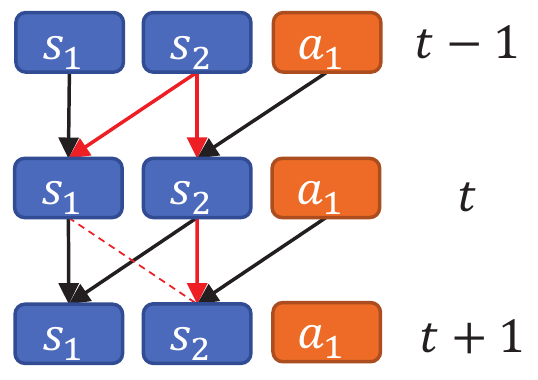}
          \end{minipage}
          \label{fig:spurious_edges:history}
    }
    \caption{two examples of spurious edges. The true causal relations are shown in solid arrows, where the backdoor paths mentioned are highlighted in red. The spurious edges are shown in red dashed lines.}
    \label{fig:spurious_edges}
 \end{figure}

Therefore, we detect causal dependencies using conditional independent tests (CITs), where the condition variables block all the backdoor paths (if there exist any). We assume that state variables transit and outcome variables arise independently, which brings several benefits: 1) The causal graph is bipartite, making our model able to be computed in parallel. 2) The causal graph can be uniquely identified using CITs. Wang \etal \shortcite{wang_causal_2022} adopt a similar approach for causal discovery from the perspective of conditional mutual information. Here, we describe and prove a theorem that is used to discover sound causal graphs in our approach.

\begin{theorem} [Causal Discovery for Factorized MDP] \label{theorem:causal_discovery_fmdp}
    Assume $\bm{u} = (\bm{s},\bm{a})$ and $\bm{v} = (\bm{s}', \bm{o})$ are respectively the sets of exogenous and endogenous variables in a Factorized MDP. Assume that $\Pr$ is a probability function of these variables $\bm{\delta} = (\bm{u},\bm{v})$ such that $\Pr$ is consistent with the MDP; that is, we have
    \begin{equation} \label{eq:consistent_pr}
        \Pr(\bm{\delta}) = \Pr(\bm{s}) \pi(\bm{a}|\bm{s}) P(\bm{v}|\bm{u})
    \end{equation}
    where $\Pr(\bm{s})$ may follow arbitrary state distribution, $\pi$ can be arbitrary policy, and $ P(\bm{v}|\bm{u})$ is the transition probability of the MDP.
    Then, there always exists a DAG $\mathcal{G}$ that represents $\Pr$. Further, if we make the following assumptions:
    \begin{enumerate}
        \item (Independent Transition) next-state variables and outcome variables are produced independently, given by
        $$
            P(\bm{v}|\bm{u}) = \prod_{j=1}^{n_s} P(v_j|\bm{u}) \prod_{k=1}^{n_o} P(o_k|\bm{u});
        $$
        \item (Causal Faithfulness) $\Pr$ is faithful to $\mathcal{G}$,
    \end{enumerate}
    then the following propositions about $\mathcal{G}$ hold:
    \begin{enumerate}
        \item No lateral edge like $v_i \rightarrow v_j$ exists among $\bm{v}$. In other worlds, we have $PA(v_j) \subseteq \bm{u}$ for every $v_j \in \bm{v}$.
        \item The parent sets of $\bm{v}$ are uniquely identified by
        $$
            (u_i \upmodels_{\Pr} v_j | \bm{u}_{-i}) \Leftrightarrow u_i \in PA(v_j)
        $$
        for every $i=1,\cdots n_s+n_a$ and every $j=1,\cdots n_s + n_o$. Here $\bm{u}_{-i}$ denotes $\bm{u} \setminus \{u_i\}$
        \item The parent sets of $\bm{v}$ in $\mathcal{G}$ is invariant. That is, we may replace $\Pr$ with any other probability function $\Pr_*$ that satisfies Eq. \ref{eq:consistent_pr} and obtain a new DAG $\mathcal{G}_*$ that represents $\Pr_*$. Assuming $\Pr_*$ is faithful to $\mathcal{G}_*$, the parent sets of $\bm{v}$ will stays unchanged:
        $$
            PA(v_j) = PA_*(v_j),\quad j=1,\cdots,n_s+n_o
        $$
    \end{enumerate}
\end{theorem}

\begin{proof}
We first prove the existence of $\mathcal{G}$. Since $\pi$ and $P$ in Eq. \ref{eq:consistent_pr} are both conditional probability function, we write that
$$
    \Pr(\bm{\delta}) = \Pr(\bm{s}) \Pr(\bm{a}|\bm{s}) \Pr (\bm{v}|\bm{u})
$$
where $\Pr(\bm{a}|\bm{s}) = \pi(\bm{a}|\bm{s})$ and $\Pr (\bm{v}|\bm{u}) = P(\bm{v}|\bm{u})$. Using the chain rule of probability functions, we have
\begin{align*}
    \Pr(\bm{s}) =& \Pr(s_1)\Pr(s_2|s_1)\cdots \Pr(s_{n_s}|s_1,\cdots,s_{n_s-1}) \\
    \Pr(\bm{a}|\bm{s}) =& \Pr(a_1|\bm{s})\Pr(a_2|a_1,\bm{s})\cdots \Pr(a_{n_a}|a_1,\cdots,a_{n_a-1},\bm{s})\\
    \Pr(\bm{v}|\bm{u}) =& \Pr(v_1|\bm{u})\Pr(v_2|v_1,\bm{u})\cdots \\
    & \qquad \Pr(v_{n_s+n_o}|v_1,\cdots,v_{n_s+n_o-1},\bm{u})
\end{align*}
We define
$$PA(s_i) \subseteq (s_1, \cdots, s_{i-1}),\quad i=1,\cdots,n_s$$
as any subset such that 
$$
\Pr(s_i|s_1,\cdots,s_{i-1}) = \Pr(s_i|PA(s_i)).
$$
Similarly, we may define
\begin{align*}
    PA(a_i) \subseteq & (\bm{s},a_1, \cdots, a_{i-1}), & i=1,\cdots,n_a;\\
    PA(s_i') \subseteq & (\bm{u},s_1', \cdots, s_{i-1}'), & i=1,\cdots,n_s; \\
    PA(o_i) \subseteq & (\bm{u},\bm{s}', o_1, \cdots, o_{i-1}), & i=1,\cdots,n_o.
\end{align*}
Together, we will have that
\begin{equation*}\begin{aligned}
    \Pr(\bm{\delta}) =& \prod_{i=1}^{n_s}\Pr(s_i|PA(s_i))
        \prod_{j=1}^{n_a}\Pr(a_j|PA(a_j)) \\
        &\prod_{k=1}^{n_k}\Pr(s_k'|PA(s_k'))
        \prod_{l=1}^{n_o}\Pr(o_l|PA(o_l))\\
    =& \prod_{x \in \bm{\delta}}\Pr(x|PA(x)).
\end{aligned}\end{equation*}
Letting the edges in $\mathcal{G}$ be given by the parent sets defined above, it is obvious that $\mathcal{G}$ represents $\mathcal{Pr}$.

Now we assume that the independent transition and causal faithfulness hold.

Assume that $\mathcal{G}$ contains a lateral edge like $v_1 \rightarrow v_2$ for example. According to independent transition, we have $(v_1 \upmodels_{\Pr} v_2 | \bm{u})$. Because $v_1 \in PA(v_2)$, we have $(v_1 \nupmodels_{\mathcal{G}} v_2 | \bm{u})$. This violates the assumption of causal faithfulness, as we have
$$
(v_1 \upmodels_{\Pr} v_2 | \bm{u}) \not\Rightarrow (v_1 \upmodels_{\mathcal{G}} v_2 | \bm{u}).
$$
Therefore, we prove that no lateral edge among $\bm{v}$ exists in $\mathcal{G}$.

Because there is no lateral edge among $\bm{v}$ exists in $\mathcal{G}$, $\bm{u}_{-i}$ blocks all paths form $u_i$ to $v_j$ unless $u_i\in PA(v_j)$. Therefore, it is easy to prove that
$$
(u_i \nupmodels_{\mathcal{G}} v_j | \bm{u}_{-i}) \Leftrightarrow u_i\in PA(v_j)
$$
Combining Theorem \ref{theorem:d-separation} and Definition \ref{def:faithfulness}, we have that
$$
(u_i \nupmodels_{\mathcal{G}} v_j | \bm{u}_{-i}) \Leftrightarrow
(u_i \nupmodels_{\Pr} v_j | \bm{u}_{-i})
$$
Therefore, the parents of $v_j$ are uniquely identified by the rule:
$$
(u_i \nupmodels_{\Pr} v_j | \bm{u}_{-i}) \Leftrightarrow u_i\in PA(v_j)
$$

Finally, let us consider another probability function $\Pr_*$ that satisfies Eq. \ref{eq:consistent_pr}, and assume $\mathcal{G}$ is the DAG that $\Pr_*$ is compatible with and faithful to. For every $v_j \in \bm{v}$ it follows that
$$\Pr(v_j|\bm{u}) = \Pr_*(v_j|\bm{u}) = P(v_j|\bm{u})$$
If $u_i \in PA(v_j)$ and $u_i \not\in PA_*(v_j)$, using the above rule we have $(u_i \nupmodels_{\Pr} v_j | \bm{u}_{-i})$ and $(u_i \upmodels_{\Pr_*} v_j | \bm{u}_{-i})
$. In other words, we have
\begin{align*}
    \Pr(v_j|\bm{u}) &\neq \Pr(v_j|\bm{u}_{-i}) \\
    \Pr_*(v_j|\bm{u}) &= \Pr_*(v_j|\bm{u}_{-i})
\end{align*}
This leads to that
$$
\Pr(v_j|\bm{u}_{-i}) \neq \Pr_*(v_j|\bm{u}_{-i})
$$
We can also write that
\begin{align*}
    ~& \Pr(v_j|\bm{u}_{-i}) \\
    =& \int_{u_i} \Pr(v_j|\bm{u}) \Pr(u_i|\bm{u}_{-i}) \\
    =& \int_{u_i} \Pr_*(v_j|\bm{u}) \Pr(u_i|\bm{u}_{-i}) \\
    =& \int_{u_i} \Pr_*(v_j|\bm{u}_{-i}) \Pr(u_i|\bm{u}_{-i}) \\
    =& \Pr_*(v_j|\bm{u}_{-i}) \int_{u_i}  \Pr(u_i|\bm{u}_{-i}) \\
    =& \Pr_*(v_j|\bm{u}_{-i}).\\
\end{align*}
From the above equations, we obtain the paradox that
$$\Pr(v_j|\bm{u}_{-i}) = \Pr_*(v_j|\bm{u}_{-i}).$$
Using reduction to absurdity, we obtain that $u_i \in PA(v_j)$ implies $u_i \in PA_*(v_j)$. Similarly, we can prove the opposite direction of this implication. As a result, we have
$$u_i \in PA(v_j) \Leftrightarrow u_i \in PA_*(v_j),$$
which shows that $PA(v_j) = PA_*(v_j)$.

Proof ends.
\end{proof}
\section{Explanation through Causal Chains} \label{appendix:explain}

In our paper, we only present the visualization of causal chains. Although this visualization offers a certain extent of interpretability, Madumal et al \shortcite{madumal_explainable_2020} have proposed techniques to better use causal chains to generate high-quality explanations. In this section, we introduce how our causal chains adapt to their techniques.

Consider an $H$-step trajectory $(\bm{\delta}^t, \bm{\delta}^{t+1}, \cdots, \bm{\delta}^{t+H-1})$, where $\bm{\delta}^{t+k} = (\bm{s}^{t+k},\bm{a}^{t+k},\bm{s}^{t+k+1},\bm{o}^{t+k},r^{t+k})$. In this trajectory, we use the bold capital letter $\bm{C}$ to denote the sub-set of variables in the causal chain. In Madumal's work, an explanation is derived from an ``explanan", which contains information about how action leads to rewards.

\begin{definition}[Explanan] A $H$-step \textit{explanan} for an action $\bm{a}^t$ under a factual trajectory $(\bm{\delta}^t, \cdots, \bm{\delta}^{t+H-1})$ is a tuple $\langle \bm{x}_r, \bm{x}_h, \bm{x}_i \rangle$, where
\begin{enumerate}
    \item $\bm{x}_r$ contains the reward variables in the causal chain.
    \item $\bm{x}_h = \bm{s}^t \cap \bm{C}$ is the heading variables (the state variables at the beginning) in the causal chain.
    \item $\bm{x}_i\subset \bm{C} \setminus (\bm{x}_h, \bm{x}_i)$ contains some intermediate variables between $\bm{x}_h$ and $\bm{x}_r$.
\end{enumerate}
\end{definition}

An explanation for ``why the agent took action $\bm{a}^t$ at $\bm{s}^t$" is generated by filling the values in the explanan into a natural-language template. If $\bm{x}_i$ contains all intermediate variables, the explanan is called a \textit{complete explanan}. However, it may contain too much information and difficult to be understood.  Therefore, Madumal et al suggest using the \textit{minimally complete explanan} (MCE), where $\bm{x}_i$ contains only the parents of $\bm{x}_r$.

Taking the causal chain in the paper's Figure \ref{fig:buildmarine_chain} for example, we present the explanations respectively drawn from the complete explanan and the minimally complete explanan below.

\begin{example}[Complete Explanation for Build-Marine] The agent build supplies depots because this action causes the following changes:
\begin{enumerate}
    \item Instantly, the number of supply depots increases from 7 to 11, and money decreases from 505 to 330;
    \item After 2 steps, the number of barracks increases from 1 to 3, and the amount of money increases from 330 to 465.
    \item After 3 steps, the number of barracks increases from 3 to 6, and the amount of money decreases from 330 to 430.
    \item After 4 steps, the number of marines increases from 0 to 6.
\end{enumerate}
Which lead to a reward of $6$ due to new marines after 4 steps.
\end{example}

\begin{example}[Minimally Complete Explanation for Build-Marine] The agent builds supplies depots because this action would eventually cause the number of marines to increase from 0 to 6 after 4 steps, which leads to a reward of $6$ due to new marines.
\end{example}

In addition, by comparing two MCEs, we can construct contrastive explanations that answer why that agent did not take another action. Therefore, Madumal et al define the minimally complete contrastive explanation, which contains the difference between the factual MCE and the counterfactual MCE.

\begin{definition}[Minimally Complete Contrastive Explanation] Let $(\bm{\delta}^t, \cdots, \bm{\delta}^{t+H-1})$ denote the factual trajectory (the trajectory that actually happens) and $(\tilde{\bm{\delta}}^t, \cdots, \tilde{\bm{\delta}}^{t+H-1})$ denote the counterfactual trajectory, which is produced by replacing $\bm{a}^t$ with another action $\tilde{\bm{a}}^t$ and using the world model and policy to simulate the following $H$ steps. Assume that $\bm{x} = \langle \bm{x}_r, \bm{x}_h, \bm{x}_i \rangle$ is the MCE for $\bm{a}^t$ under the factual trajectory and that $\bm{y} = \langle \bm{y}_r, \bm{y}_h, \bm{y}_i \rangle$ is the MCE for $\tilde{\bm{a}}^t$ under the counterfactural trajectory. A \textit{minimally complete contrastive explanation} (MCCE) is then given by a tuple $\langle \bm{x}^{diff}, \bm{y}^{diff}, \bm{x}_r \rangle$, where
\begin{enumerate}
    \item $\bm{x}^{diff}$ is the subset of variables in $\bm{x}$ that 1) is not included in $\bm{y}$, or 2) owns a value different from that in $\bm{y}$.
    \item $\bm{y}^{diff}$ is the subset of variables in $\bm{y}$ that 1) is not included in $\bm{x}$, or 2) owns a value different from that in $\bm{x}$.
    \item $\bm{x}_r$ contains the reward variables in the factual causal chain.
\end{enumerate}
\end{definition}
\section{The Trade-off between interpretability and accuracy}
\label{appendix:proof_tradeoff}

\begin{theorem}
    Assume $v_j$ is an endogenous variable of the SCM of the Factorized MDP. Let $PA^1(v_j)$ and $PA^2(v_j)$ respectively denote its parent sets discovered using the threshold $\eta_1$ and $\eta_2$. Assume $PA^*(v_j)$ is the ground-truth parent set of $v_j$. If $\eta_1 \leq \eta_2$, with the well trained structural equation $f^j$ we have
    \begin{equation}
        \begin{aligned}
            \mathbbm{E}_{\bm{u}} & \left[ D_{KL}  \left(f^j(PA^*(v_j)) || f^j(PA^1(v_j))\right) \right] \geq \\
             & \mathbbm{E}_{\bm{u}} \left[ D_{KL}\left(f^j(PA^*(v_j))||f^j(PA^2(v_j))\right) \right]
        \end{aligned}
    \end{equation}
\end{theorem}

\begin{proof}
    Obviously, we have $PA^1(v_j) \subseteq PA^2(v_j)$. For simplicity, we define
    \begin{align*}
        \bm{a} &:= PA^1(v_j) \cap PA^*(v_j),\\
        \bm{b} &:= (PA^2(v_j) \cap PA^*(v_j)) \setminus PA^1(v_j),\\
        \bm{c} &:= PA^*(v_j) \setminus PA^2(v_j),\\
        \bm{d} &:= PA^1(v_j) \setminus \bm{a}, \\
        \bm{e} &:= PA^2(v_j) \setminus (\bm{a, b, d}),
    \end{align*}
    where $\bm{a,b,c,d,e}$ are non-overlapping subsets of $\bm{u}$. More specifically, $\bm{a,b,c}$ are the true parents of $v_j$, whereas $\bm{d,e}$ are false parents of $v_j$. Then we have
    \begin{align*}
        PA^*(v_j) &= (\bm{a}, \bm{b}, \bm{c}),\\
        PA^1(v_j) &= (\bm{a}, \bm{d}),\\
        PA^2(v_j) &= (\bm{a}, \bm{b}, \bm{d}, \bm{e}),\\
    \end{align*}
    We use $\sim$ to denote that a probability conforms to a given distribution. Noted that variables in $\bm{d}$ are not true parents of $v_j$, with well trained $f^j$, the posterior distribution of $v_j$ is given by
    \begin{equation}
        \mathrm{Pr}(v_j|\bm{a}) = \mathrm{Pr}(v_j|\bm{a}, \bm{d}) \sim f^j(\bm{a}, \bm{d}) = f^j(PA^1(v_j))
    \end{equation}
    Therefore, we have
    \begin{equation}\begin{aligned}
         & \mathbbm{E}_{\bm{u}} \left[ D_{KL}\left(f^j(PA^*(v_j)) || f^j(PA^1(v_j))\right) \right] \\
        =& \mathbbm{E}_{\bm{u}} \left[ D_{KL}\left(f^j(\bm{a}, \bm{b}, \bm{c}) || f^j(\bm{a}, \bm{d})\right) \right] \\
        =& \int_{\bm{u}} \mathrm{Pr}(\bm{u}) \mathrm{d} \bm{u} \int_{v_j}  \mathrm{Pr}(v_j|\bm{a}, \bm{b}, \bm{c}) \log \frac{\mathrm{Pr}(v_j|\bm{a}, \bm{b}, \bm{c})}{\mathrm{Pr}(v_j|\bm{a}, \bm{d})} \mathrm{d}v_j \\
        =& \int_{\bm{a,b,c}} \mathrm{Pr}(\bm{a,b,c}) \mathrm{d}(\bm{a,b,c}) \int_{v_j}  \\ 
         & \qquad \qquad \mathrm{Pr}(v_j|\bm{a}, \bm{b}, \bm{c}) \log \frac{\mathrm{Pr}(v_j|\bm{a}, \bm{b}, \bm{c})}{\mathrm{Pr}(v_j|\bm{a})} \mathrm{d}v_j \\
        =& \int_{\bm{a,b,c},v_j} \mathrm{Pr}(\bm{a}, \bm{b}, \bm{c}, v_j) \log \frac{\mathrm{Pr}(v_j|\bm{a}, \bm{b}, \bm{c})}{\mathrm{Pr}(v_j|\bm{a})} \mathrm{d}(\bm{a,b,c},v_j)
    \end{aligned}\end{equation}
    Similarly, we have
    \begin{equation}\begin{aligned}
         & \mathbbm{E}_{\bm{u}} \left[  D_{KL}\left(f^j(PA^*(v_j)) || f^j(PA^2(v_j))\right) \right] \\
        =& \mathbbm{E}_{\bm{u}} \left[ D_{KL}\left(f^j(\bm{a}, \bm{b}, \bm{c}) || f^j(\bm{a,b,d,e})\right) \right] \\
        =& \int_{\bm{a,b,c},v_j} \mathrm{Pr}(\bm{a}, \bm{b}, \bm{c}, v_j) \log \frac{\mathrm{Pr}(v_j|\bm{a}, \bm{b}, \bm{c})}{\mathrm{Pr}(v_j|\bm{a,b})} \mathrm{d}(\bm{a,b,c},v_j)
    \end{aligned}\end{equation}
    Together, we write
    \begin{equation}\begin{aligned}
        & \mathbbm{E}_{\bm{u}} \left[ D_{KL} \left(f^j(PA^*(v_j)) || f^j(PA^1(v_j))\right) \right] \\
        & \qquad \qquad - \mathbbm{E}_{\bm{u}} \left[ D_{KL}\left(f^j(PA^*(v_j))||f^j(PA^2(v_j))\right) \right] \\
        =& \int_{\bm{a,b,c},v_j} \mathrm{Pr}(\bm{a}, \bm{b}, \bm{c}, v_j) \log \frac{\mathrm{Pr}(v_j|\bm{a}, \bm{b})}{\mathrm{Pr}(v_j|\bm{a})} \mathrm{d}(\bm{a,b,c},v_j) \\
        =& \int_{\bm{a,b},v_j} \int_{\bm{c}}(\mathrm{Pr}(\bm{a}, \bm{b}, \bm{c}, v_j)\mathrm{d}\bm{c}) \log \frac{\mathrm{Pr}(v_j|\bm{a}, \bm{b})}{\mathrm{Pr}(v_j|\bm{a})} \mathrm{d}(\bm{a,b},v_j) \\
        =& \int_{\bm{a,b},v_j} \mathrm{Pr}(\bm{a}, \bm{b}, v_j) \log \frac{\mathrm{Pr}(v_j|\bm{a}, \bm{b})}{\mathrm{Pr}(v_j|\bm{a})} \mathrm{d}(\bm{a,b},v_j) \\
        =& \int_{\bm{a,b},v_j} \mathrm{Pr}(\bm{a}, \bm{b}) \mathrm{Pr}(v_j|\bm{a}, \bm{b})  \log \frac{\mathrm{Pr}(v_j|\bm{a}, \bm{b})}{\mathrm{Pr}(v_j|\bm{a})} \mathrm{d}(\bm{a,b},v_j) \\
        =& \int_{\bm{a,b}} \mathrm{Pr}(\bm{a}, \bm{b}) D_{KL}\left(f^j(\bm{a,b})||f^j(\bm{a}))\right) \mathrm{d}(\bm{a,b})  \\
        =& \ \mathbbm{E}_{\bm{a,b}} \left[ D_{KL}\left(f^j(\bm{a,b})||f^j(\bm{a}))\right) \right] \\
        \geq & \ 0
    \end{aligned} \end{equation}
\end{proof}
\section{The Algorithm for Model-Based RL} \label{appendix:alg_mbrl}

See Algorithm \ref{alg:mbrl}.

\begin{algorithm}[h]
\caption{Model-based RL using causal world models}
\begin{algorithmic}[1]
    \label{alg:mbrl}
    \STATE{Initialize environment models and the policy $\pi(\bm{a}|\bm{s})$}
    \FOR{training epoch $i = 1,...,n_{epoch}$}
        \STATE{Interact with the environment using $\pi$; Add transitions into buffer $\mathcal{D}$}
        \STATE (for every $n_{graph}$ epochs) Update the causal graph through causal discovery
        \STATE Fit the structural equations by maximizing (\ref{eq:loss})
        \FOR{learning round $j = 1,...,n_{round}$}
            \STATE{Generate simulated data by performing $k$-step model rollout using actor $\pi$}
            \STATE{Update the policy $\pi$ using PPO algorithm}
        \ENDFOR
    \ENDFOR
\end{algorithmic}
\end{algorithm}
\section{Environments}

\subsection{Factorization of Public Environments}

This section describes how environments are considered Factorized MDPs in our implementation.

\paragraph{Cartpole} The state is factorized into 4 variables $(x, \dot{x}, \theta, \dot\theta)$ indicating the 1) position of the cart, 2) velocity of the cart, 3) angle of the pole, and 4) angle velocity of the pole. The action only contains one discrete variable indicating the direction (left or right) to push the cart. The goal of the agent is to keep the pole upright as long as possible. Therefore, the agent is rewarded by $1$ for each step as long as the state satisfies $-2.4 \leq x \leq 2.4$ and $-12^\circ \leq \theta \leq 12^\circ$. The episode terminates if this condition does not hold. In addition, no outcome variable is included in this environment.

\textbf{Lunarlander-Discrete} The state is factorized into 7 variables $(x, y, \dot{x}, \dot{y}, \theta, \dot{\theta}, legs)$ indicating the (1) horizontal position, (2) vertical position, (3) horizontal velocity, (4) vertical velocity, (5) angle, (6) angle velocity, and (7) whether the two legs are in contact with the ground. The action contains one discrete variable indicating the engine (none, main, left, or right) to be actuated. The environment contains 3 outcome variables, including (1) the fuel cost due to firing the engine, (2) whether the lander crashed, and (3) whether the rocket is resting. The agent is rewarded (or penalized) from multiple sources, including (1) shortening the distance to the destination, (2) reducing the velocity, (3) reducing the angle, (4) increasing the number of landed legs, (5) being resting, and (6) crashing.

\textbf{Lunarlander-Continuous} replaces the action space of Lunarlander-Discrete with a continuous action space. The new action space includes 2 continuous variables ranged in $(-1, 1)$, respectively controlling the throttles of the main and the lateral engines.

\textbf{Build Marine} The state includes 6 variables denoted as $(n_{wk}, n_{mr}, n_{br}, n_{dp}, money, time)$, namely the (1) number of workers, (2) number of marines, (3) number of barracks, (4) number of supply depots, (5) amount of money, and (6) game time. The action includes only one discrete variable indicating which unit (workers, marines, barracks, supply depots, or none) to be built. The player is rewarded with $1$ for every newly produced marine. The problem is challenging since marines can only be built from barracks, and barracks can only be built provided there exists at least one supply depot. In addition, the number of workers and marines is limited by the number of supplies provided by supply depots.

\subsection{The Environment for Measuring the Accuracy of Recovering Action Influence}
\label{appendix:env_aimtest}

To measure the accuracy of recovering the causal dependencies of the AIM, we design an additional environment with a known ground-truth AIM. The environment contains one action variable $a$ and 5 state variables: $x_1, x_2, x_3, x_4,$ and $\tau$. The action variable $a$ is chosen from 4 options $\{0, 1, 2, 3\}$. The dynamics of these state variables are given by
\begin{align}
    x_1' &= x_1 + \mathcal{N}(1, 1) \\
    x_2' &= \left\{\begin{array}{lr}
         x_1, & a = 0  \\
         x_2, & \textrm{otherwise}
    \end{array}\right. + \mathcal{N}(0, 1) \\
    x_3' &= x_3 + \left\{\begin{array}{lr}
         x_1, & a = 0  \\
         x_2, & a = 1 \\
         5, & a = 2 \\
         10, & a = 3 \\
    \end{array}\right. + \mathcal{N}(0, 1) \\
    x_4' &= 0.1 x_3 +  0.9 x_4 + \mathcal{N}(0, 0.5) \\
    \tau' &=  \tau + \left\{\begin{array}{lr}
         10, & a = 0  \\
         20, & a = 1 \\
         5, & a = 2 \\
         5, & a = 3 \\
    \end{array}\right.
\end{align}

The ground-truth causality of the AIM can be easily derived from the above dynamics. Aiming to confuse neural networks, these dynamics create copious spurious correlations in the data. For example, $x_1$ and $\tau$ present a strong positive correlation, whereas neither of them is the other's causation. 
\section{Importance of Causality: an Example} \label{appendix:causality_case_study}

To show how causal discovery facilitates explanation generation, we investigated the causal chains produced without a discovered causal graph. Figure \ref{fig:chain_full} presents such a causal chain for Build-Marine. It is produced by our model using a full graph, where unreasonable connections are highlighted in bold and red. We found the model attributes a large influence weight ($\approx 0.78$) to $time$ as the salient parent of $n_{depot}$ under the action of building supply depots. This leads to the incorrect understanding that ``if we build supply depots, the number of supply depots turns to $4$ because the game time is $140$". In fact, the variable $time$ only influences $time'$ in the discovered causal graph presented in Figure \ref{fig:buildmarine_graph}, meaning the number of supply depots is not caused by the game time at all. The model is misled by the fact that the agent always builds more supply depots as time goes on, which makes the two variables ($time$ and $n_{depot}'$) highly related. Similarly, the model wrongly takes $time$ as the salient parent of $money$ under the actions of building barracks since it observes that money usually accrues with time.

\begin{figure}[h]
    \centering
    \includegraphics[width=8.5cm]{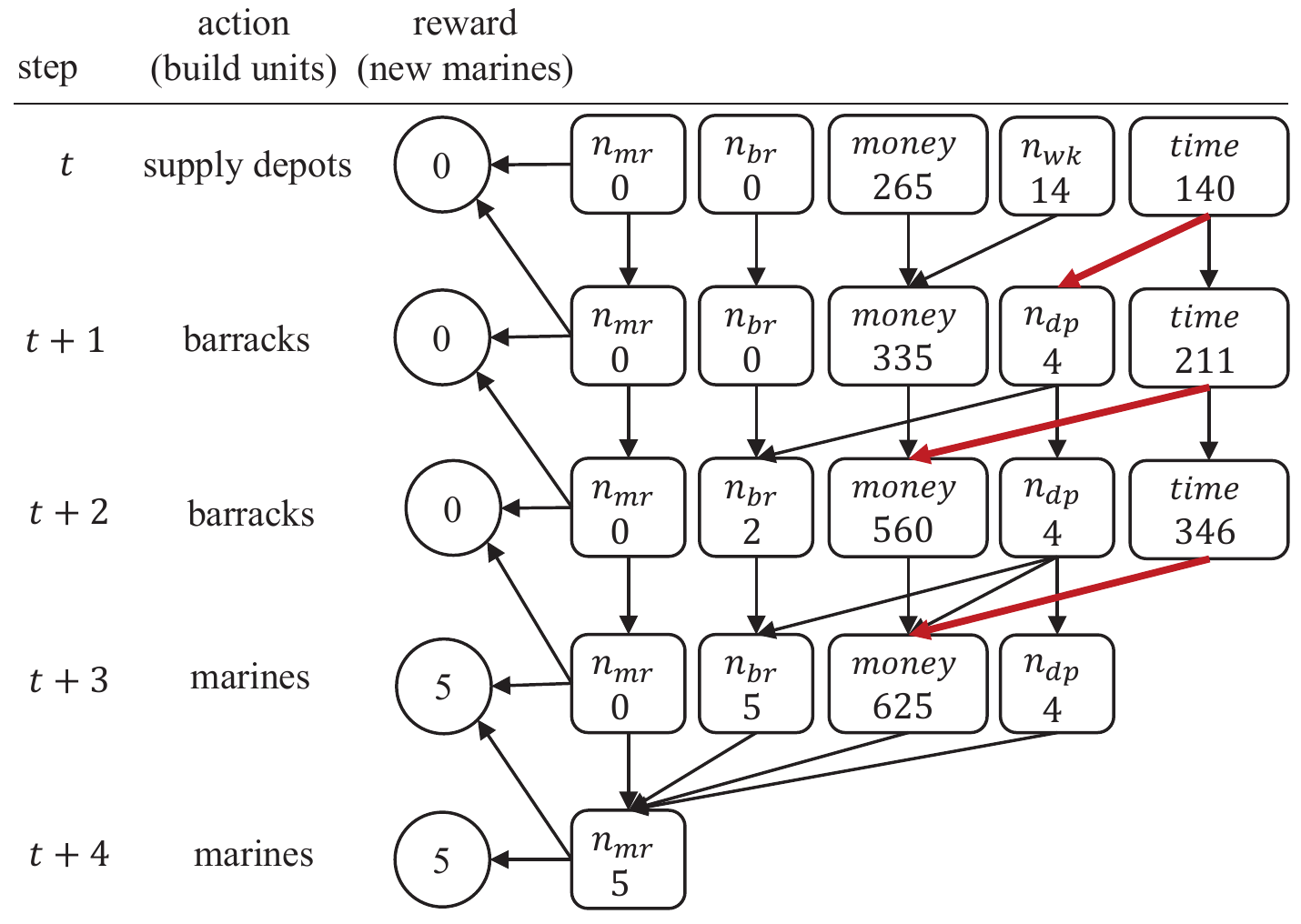}
    \caption{An example of the causal chain produced by a full causal graph. We highlight the ``problematic" edges in bold and red.}
    \label{fig:chain_full}
\end{figure}

Therefore, attention alone is insufficient to obtain reasonable causal chains for explanations, as it can be easily misled by spurious correlations (which are rife in the non-i.i.d. data collected in RL). Fortunately, these errors in causal chains are greatly reduced when combining attention with a causal graph, which precludes most spurious correlations and leads to correct causal chains.

\section{Hyper-parameters of Model-based RL}

\begin{table*}[tb]
    \centering
    \begin{tabular}{l|cccc}
        \hline
        Parameter  & Cartpole & LunarLander & LunarLander-Continuous & Build-Marine \\
        \hline
        total epochs ($n_{epoch}$) & 50 & 200 & 200 & 200 \\
        environment steps per epoch  & 800 & 2400 & 2400 & 120 \\
        policy-update rounds per epoch ($n_{round}$) & 20 & 20 & 20 & 20 \\
        epochs per graph update ($n_{graph}$)  & $3\rightarrow 5$ & $3\rightarrow 5$ & $3\rightarrow 5$ & $3\rightarrow 5$  \\
        rollout length & $1\rightarrow 5$ & $1\rightarrow 5$ & $1\rightarrow 5$ & $5\rightarrow 10$\\
        rollout samples for policy update & 4096 & 4096 & 4096 & 2048 \\
        causal threshold ($\eta$) & 0.1 & 0.3 & 0.4 & 0.15 \\
        model-ensemble size & 5 & 5 & 5 & 5\\
        discount factor ($\gamma$) & 0.98 & 0.975 & 0.975 & 0.97 \\
        \hline
    \end{tabular}
    \caption{main parameters used in model-based RL. The form ``$a\rightarrow b$" denotes the parameter gradually changes from $a$ to $b$ during the training process.}
    \label{tab:hyper-parameters}
\end{table*}

The main hyper-parameters used in the mentioned environments for model-based RL are presented in Table \ref{tab:hyper-parameters}.
\section{Computational Complexity}
Let $n = max(n_s + n_a, n_s + n_o)$ roughly denote the number of variables of the environment. Let $N$  denote the total number of transition samples. Let $b$ denote the batch size.

\paragraph{Model} The parameter scale of our model is $O(n)$. The time complexity of one forward pass is $O(n^2b)$.

\paragraph{Causal discovery} The time complexity of testing each edge through FCIT is $O(n N \log N)$ and the overall time complexity of causal discovery is $O(n^3 N \log N)$. The space complexity for causal discovery is $O(nN)$ if the tests are sequentially performed.

\paragraph{Causal chains} The time complexity and space complexity of generating an $H$-step causal chain are both $O(n^2H)$. In our experiments, the generation of each causal chain completes almost instantly.

\end{document}